\documentclass[twoside,web]{ieeecolor}
\usepackage{tmi}
\usepackage{cite}
\usepackage{amsmath,amssymb,amsfonts}
\usepackage{algorithmic}
\usepackage{graphicx}
\usepackage{subcaption}
\usepackage{textcomp}
\usepackage{array}

\usepackage{booktabs}
\usepackage[normalem]{ulem}
\useunder{\uline}{\ul}{}
\usepackage{makecell}
\usepackage{multirow}
\usepackage{stfloats}

\usepackage{amsmath,amsfonts,bm}

\def\eqref#1{equation~\ref{#1}}

\def\1{\bm{1}}

\def\rmX{{\mathbf{X}}}
\def\rmY{{\mathbf{Y}}}
\def\rmZ{{\mathbf{Z}}}

\DeclareMathAlphabet{\mathsfit}{\encodingdefault}{\sfdefault}{m}{sl}
\SetMathAlphabet{\mathsfit}{bold}{\encodingdefault}{\sfdefault}{bx}{n}

\newcommand{\KL}{D_{\mathrm{KL}}}

\usepackage{colortbl}
\definecolor{mygray}{gray}{1}
\definecolor{mylightgray}{gray}{1}
\newtheorem{proposition}{Proposition}

\usepackage{color}
\newcommand\red{\textcolor{red}}

\newcommand\blue{\textcolor{black}}

\def\BibTeX{{\rm B\kern-.05em{\sc i\kern-.025em b}\kern-.08em
    T\kern-.1667em\lower.7ex\hbox{E}\kern-.125emX}}
\begin{document}
\title{MedMAP: Promoting Incomplete Multi-modal Brain Tumor Segmentation with Alignment}
\author{Tianyi Liu, Zhaorui Tan, Muyin Chen, Xi Yang, Haochuan Jiang, Kaizhu Huang
\thanks{Tianyi Liu is with the University of Liverpool and
School of Robotics, XJTLU Entrepreneur College (Taicang), Xi’an Jiaotong-Liverpool University;
Zhaorui Tan is with the University of Liverpool and
School of Advanced Technology, Xi’an Jiaotong-Liverpool University; 
Muyin Chen and Haochuan Jiang are with the School of Robotics, XJTLU Entrepreneur College (Taicang), Xi’an Jiaotong-Liverpool University;
Xi Yang is with School of Advanced Technology, Xi’an Jiaotong-Liverpool University;
Kaizhu Huang is with Data Science Research Center, Duke Kunshan University}
\thanks{The work was partially supported by the following: National Natural Science Foundation of China under No.92370119, No. 62206225, and No. 62376113; 
Jiangsu Science and Technology Program (Natural Science Foundation of Jiangsu Province) under No. BE2020006-4;
Natural Science Foundation of the Jiangsu Higher Education Institutions of China under No. 22KJB520039;
XJTLU Research Development Funding 20-02-60. 
Computational resources used in this research are provided by the School of Robotics, XJTLU Entrepreneur College (Taicang), Xi'an Jiaotong-Liverpool University.}
}

\maketitle

\begin{abstract}
Brain tumor segmentation is often based on multiple magnetic resonance imaging (MRI). However, in clinical practice, certain modalities of MRI may be missing, which presents a more difficult scenario. To cope with this challenge, Knowledge Distillation, Domain Adaption, and Shared Latent Space have emerged as commonly promising strategies. However, recent efforts typically overlook the modality gaps and thus fail to learn important invariant feature representations across different modalities. Such drawback consequently leads to limited performance for missing modality models. 
{To ameliorate these problems, pre-trained models are used in natural visual segmentation tasks to minimize the gaps. However, promising pre-trained models are often unavailable in medical image segmentation tasks.}
{Along this line,}
in this paper, we propose a novel paradigm that aligns latent features of involved modalities to a well-defined distribution anchor as the {substitution of the pre-trained model}. As a major contribution, we prove that our novel training paradigm ensures a tight evidence lower bound, thus theoretically certifying its effectiveness. Extensive experiments on different backbones validate that the proposed paradigm can enable invariant feature representations and produce models with narrowed modality gaps. Models with our alignment paradigm show their superior performance on both BraTS2018 and BraTS2020 datasets.
\end{abstract}


\section{Introduction}
\label{sec:introduction}
\IEEEPARstart{B}{rain} tumors 
pose severe risks to human life, making precise medical segmentation essential as it devises effective treatment planning and strategies~\cite{chen2021learning}.
Brain tumor segmentation methods usually employ
Multiple Magnetic Resonance Imaging (MRI) visualizations, including Fluid Attenuation Inversion Recovery (Flair), contrast-enhanced T1-weighted (T1ce), T1-weighted (T1), and T2-weighted (T2), as illustrated in Figure~\ref{fig:banner_modality}, as multiple modalities~\cite{zhao2022modality}.
Particularly,
the aforementioned modalities complement each other to understand both the physical structure and physiopathology of tumors; their combination naturally leads to improved segmentation performance~\cite{lindig2018evaluation,menze2014multimodal,patil2013medical,maier2017isles}. 
For instance, T2 and Flair imaging modalities are valuable for assessing the enhanced tumor (ET), whereas T1 and T1ce modalities are effective in delineating the core of the tumor, 
including the necrotic, non-enhancing, and enhancing regions (NCR/NET)~\cite{ding2021rfnet}. 
However, not all modalities are always available in real clinical practice due to scenarios such as data corruption and/or variations in scanning protocols, leading to crucial challenges of developing a generalized multi-modal approach that copes with the absence of certain modalities~\cite{chen2023query,qiu2023scratch,qiu2023modal,liu2021incomplete}.
\begin{figure}
    \centering
    \includegraphics[width=\linewidth]{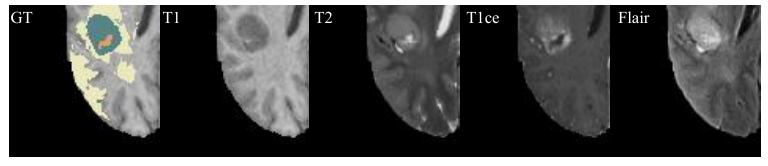}
    \caption{Images of the four modalities in the brain tumor dataset with the ground truth (GT) segmentation label. Different colors represent different organs: Orange: NCR/NET, Yellow: ED, and Blue: ET.}
    \label{fig:banner_modality}
\end{figure}

In response to the challenge of missing modalities, previous substantial efforts have converged on two primary strategies: Knowledge Distillation (KD), and Domain Adaptation (DA). The KD-based attempts involve transferring the knowledge from the complete multi-modality teacher to the incomplete missing-modality student~\cite{hu2020knowledge,chen2021learning,wang2023prototype, azad2022smu}, while the methods in DA leverage alignment methods to bridge the gap between the model trained on complete modalities and those trained on incomplete ones~\cite{wang2021acn}. 
However, these two strategies ignore gaps existing between the modalities as they often process different modalities as separate channels, leading to a failure to capture the important invariance across modalities. 
To better alleviate the modality gaps, 
another branch of methods, named Shared Latent Space (SLS), transfers each modality into a common representation space shared among all modalities~\cite{havaei2016hemis,varsavsky2018pimms,zhang2022mmformer,ding2021rfnet}.
It argues that they can learn the inter-modality invariance~\cite{zhang2022mmformer}, but the detailed analysis of minimized modality gaps seems missing. 
As illustrated in Figure.~\ref{fig:banner_tsne_visual}, we reveal that the modality gaps still exist in most previous methods. Particularly, SLS-based methods have been observed to alleviated the modality gaps, but these gaps have not been sufficiently eliminated.

{
This paper mainly studies the critical 
but an under-explored question for missing modality segmentation: 
Can minimizing modality gaps improve the model's generalization and lead to better missing modality performance? 
Our extensive experiments unveil that minimizing modality gaps leads to consistent improvements in various missing modalities approaches, including those following KD, DA, and SLS strategies. 
Furthermore, to achieve a convincing minimization of modality gaps for different approaches, we propose the \textbf{Med}ical \textbf{M}odality \textbf{A}lignment \textbf{P}aradigm (MedMAP), which is feasible for most existing missing modality approaches.}

The design of MedMAP is inspired by minimizing domain gaps for natural visual segmentation~\cite{choi2021robustnet,tan2024rethinking},  which consistently yields improvements across various approaches and often relies on pre-trained models trained by abundance data.  
Specifically, the distribution of latent features from the pre-trained model is used as the alignment anchor to minimize domain gaps, providing stable and reliable alignment guidance.  
However, promising pre-trained models for medical tasks, especially brain tumor segmentation tasks, are often unavailable due to a lack of data.

To tackle this issue, MedMAP includes a pre-defined distribution, $P_{mix}$, as the substitution of the pre-trained model. 
Especially,
MedMAP incorporates a feature-encoding pipeline
$\mathcal{T}$ to map the latent features of different modalities into a shared space, and the latent distributions of these modalities are aligned to the pre-defined $P_{mix}$. 
Theoretically, we show that individually aligning each modality to $P_{mix}$ in MedMAP achieves better performance than mapping them collectively to $P_{mix}$, reducing the modality gaps in a better manner that benefits the downstream prediction tasks in missing-modality scenarios.
Besides, the efficacy of MedMAP with two concrete proposed forms of $P_{mix}$ is empirically validated with different approaches.
Extensive experimental results
demonstrate that
MedMAP can reduce the modality gap and yield substantial benefits in a range of medical tasks where different modalities are absent. 
The major contributions of the paper are summarized as follows:
\begin{itemize}
\item We proposed a pre-defined distribution $P_{mix}$, which is part of the \textbf{M}edical \textbf{M}odality
\textbf{A}lignment \textbf{P}aradigm (MedMAP), as the substitution of the pre-trained model in the missing modality scenario.
\item We provide theoretical support for our MedMAP, showing that individually aligning each modality to $P_{mix}$ certifies tighter Evidence Lower Bound than mapping all modalities as a whole to $P_{mix}$.
\item  We empirically verify the effectiveness of our MedMAP in enhancing brain tumor segmentation performance when using the latest state-of-the-art backbones in scenarios where certain modalities are missing. It substantially mitigates modality gaps and properly learns the cross-modality invariance.
\end{itemize}

\begin{figure}
    \centering
    \includegraphics[width=\linewidth]{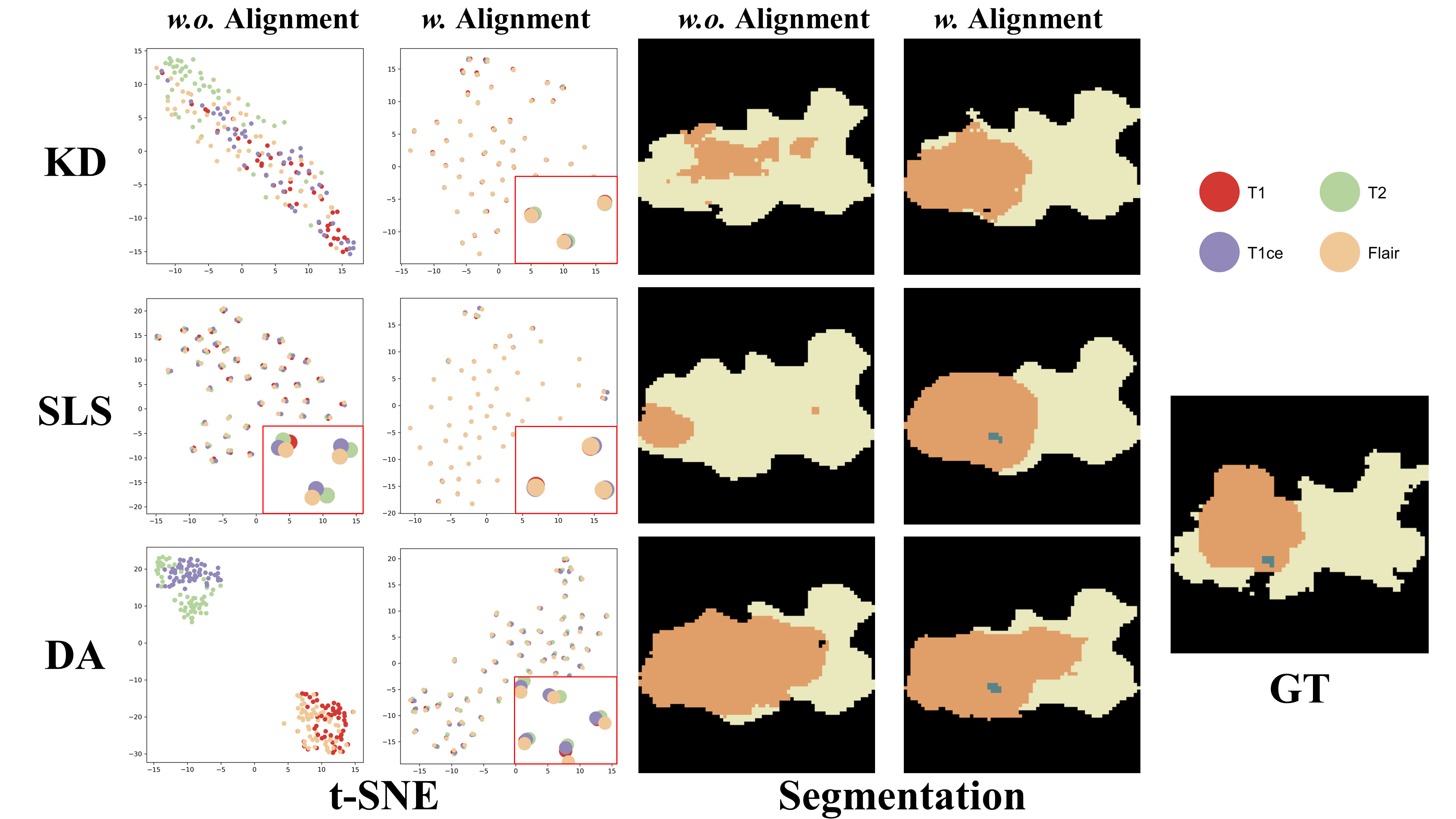}
    \caption{T-SNE and segmentation visualization of different strategies with (\textit{w.}) and without (\textit{w.o.}) alignment. 
    Different colors in t-SNE represent different modalities. GT denotes the groundtruth label.}
    \label{fig:banner_tsne_visual}
\end{figure}

The rest of this article is organized as follows. Section~\ref{sec:review} reviews related works on brain tumor segmentation with incomplete modalities. Section~\ref{sec:theory} provides the theoretical motivations of our alignment paradigm. Section~\ref{methodology} describes the alignment paradigm in detail. We evaluate our alignment paradigm with state-of-the-art methods and provide the ablation studies in Section~\ref{sec:exp}.
The article is concluded in Section~\ref{conclusion}.

\section{Related Work}
\label{sec:review}
\subsection{Multi-modal Learning for Missing Modalities}
\label{sec:reviewscope} 
Flair, T1, T2, and T1ce are complementary modalities utilized to segment brain tumors~\cite{zhao2022modality}.
Segmentation performance unfortunately drops drastically in the scenario when some modalities are missing due to practical difficulties such as data corruption and/or scanning protocol variations~\cite{chen2023query,qiu2023scratch,qiu2023modal,liu2021incomplete,ma2022multimodal}.
Prior arts have proposed several methods to deal with the missing modality problem, 
which can be generally divided into three categories: 
Knowledge Distillation, Share Latent Space, and Domain Adaption.

\textbf{Knowledge Distillation (KD)} based approaches
transfer knowledge from teachers with complete modality information to students with missing modality information.
In~\cite{chen2021learning}, Kullback-Leibler (KL) loss and an additional contrastive loss are engaged to guide the student to imitate soft distributed features from the teacher and reduce latent space divergence between them respectively.
ProtoKD~\cite{wang2023prototype} employs a prototype knowledge distillation loss to encourage simultaneous intra-class concentration and inter-class divergence.
In these methods, students' performance is always based on teachers' performance.
However, students often lead to sub-optimal segmentation performance since their teachers do not consider the feature invariant information and domain-specific information among those modalities. 

\textbf{Shared Latent Space (SLS)} methods retrieve missing information by exploiting the multimodal latent feature space.
RFnet~\cite{ding2021rfnet} uses different encoders to extract the modality-specific information and a decoder to share the weights for those modalities to build a shared representation. Different modality features are fused at different levels.
MmFormer~\cite{zhang2022mmformer} introduces Transfomers to exploit intra- and inter-modality dependency for feature fusion.
These methods presume that there are no modality gaps among these modality features and fuse them directly.  They consider modality-invariant information yet ignore modality-specific information.

\textbf{Domain Adaptation (DA)} based methods aim to minimize the gap between complete modality models and incomplete models as they share different domains.
Adversarial Co-training Network
(ACN)~\cite{wang2021acn} consists of a
multimodal path (to obtain rich modality information) and a unimodal path (to generate modality-specific feature representations). Then, a co-training approach is introduced to establish a coupled learning process between them.
However, although ACN narrows the gaps between complete and incomplete modality models, it still ignores the gaps between different modalities within complete modality models.


{This paper reveals that all the aforementioned methods are influenced by the knowledge acquired by the base model, and mitigating the modality gaps in the base model benefits all of them. 
To address this, we uniformly summarize the mathematical forms of these methods in Section~\ref{sec:theory_previous} to validate this issue.}

\subsection{Alignment in Multi-domain Generalization}
Gaps in the latent space arise not only across various modalities but also span across distinct data domains. 
 To address the domain generalization task, many attempts have been made towards a narrowed gap.
For instance, recent efforts~\cite{ganin2016domain,li2018deep, li2018domain,hu2020domain} learn domain-independent representations by reducing the gaps to improve generalization to classify images, which are, however, not directly applicable to semantic segmentation. 
In our research, we borrow the idea from the domain generalization 
to perform proper alignment between latent distributions in the \textit{teacher}. 
To this end, we successfully reduce the modality gaps to overcome the difficulty of missing modality, thereby improving the segmentation of the medical image in all categories.

\section{Theoretical Motivation}
\label{sec:theory}




In this section, we underscore the significance of reducing modality gaps that prior methodologies have overlooked. We present the problem formulations for three distinct types of approaches and illustrate the enhancement of domain alignment through the strategic use of $P_{mix}$. Furthermore, we explore different techniques for deriving practical empirical representations of $P_{mix}$.


\textbf{Notations.}
We denote the model, i.e., the teacher model for knowledge distillation, the model for shared latent space, or the model pretrained for domain adaptation as the base model.
{Considering $J$ as the number of the  set} of modalities of medical images with paired observations and targets $\{\rmX_j,\rmY_j\}_{j=1}^J$ from the modality $j$. Note for medical modalities, $\rmY$ remains static for all modalities. 
The encoders of the base model are denoted as 
$\mathcal{T}: \mathcal{T}(\rmX_j) \to \rmZ_j^*$ {where $\rmZ^*_j$ denotes the latent feature obtained from the base model of the $j^{th}$ modality.}
Simultaneously, a predictor $\mathcal{C}$ that predicts segmentation masks from $\{\rmZ_j^*\}_{j=1}^J$ as $\mathcal{C}^*:\mathcal{C}^*(\{\rmZ_j^*\}_{j=1}^J) \to \rmY$. 
Correspondingly, we denote the possible downstream model, e.g., the student models for knowledge distillation and the adapted model for the domain adaptation,
for the $j^{th}$ target modality as $\mathcal{S}_j:  \mathcal{S}_j(\rmX_{j}) \to \rmZ_{j}$ of each modality with their predictor $\mathcal{C}_{j}:\mathcal{C}(\rmZ_{j}) \to \rmY$.
Let $P(\cdot)$, $\KL(\cdot\Vert\cdot)$, $H(\cdot)$, $H_{c}(\cdot,\cdot)$, $I(\cdot;\cdot)$ denote the probability of a random variable from the distribution, Kullback–Leibler divergence, entropy, cross-entropy, and mutual information respectively.

\subsection{Previous Methods} 
\label{sec:theory_previous}

Despite the various approaches adopted to handle missing modalities, the common strategy involves training a base model $\mathcal{T}$ that retains information from all modalities as its foundation. Therefore, the objective for training $\mathcal{T}$ can be summarized as:
\begin{align}
\label{eq:t_obj_ori}
    & \max_{\mathcal{C}^*} \!\sum\nolimits_{j=1}^J \!\mathbb{E}_{\rmZ_j^*\sim P(\rmZ_j^*)}[\ln P(\rmY \! \mid \! \mathcal{C}^*(\rmZ_j^*))] + \zeta,
\end{align}
where $\zeta$  denotes any possible regularization.
$Eq.~\ref{eq:t_obj_ori}$ can be changed in the form of information entropy as:
\begin{align}      
    &\min_{\mathcal{C}^*} \sum\nolimits_{j=1}^J H_c(P( \mathcal{C}^*(\rmZ_j^*)), P(\rmY))  + \zeta.  
\end{align}
The base model can be adapted for different downstream approaches under the missing modality scenario. We provide details for each type of missing modality approaches as follows:

\textbf{KD.}
In knowledge distillation for medical segmentation with missing modality, the process involves two main steps: training the teacher and student models. As previously mentioned, the teacher model is denoted as the base model $\mathcal{T}$; its objective is in the form of $Eq.~\ref{eq:t_obj_ori}$.
Notably, previous methods~\cite{hu2020knowledge,chen2021learning,wang2023prototype} uses no extra regularization for training $\mathcal{T}$, i.e., $\zeta_{KD}:=0$.
Meanwhile, for one possible student model $\mathcal{S}_j$ encoder and $\mathcal{C}_j$ classifier trained with missing modalities that leverages knowledge from $\mathcal{T}$, its objective is: 
\begin{equation}
\label{eq:kd_obj}
    \begin{split}
    \max_{\mathcal{S}_j,\mathcal{C}_j}& 
   \mathbb{E}_{\rmZ_j\sim P(\rmZ_j)}[\ln P(\rmY \mid \mathcal{C}_j(\rmZ_j))]  \\
   &
   - \KL(P(\mathcal{S}_j(\rmX_j)) \Vert P(\rmZ_j^*)).   
    \end{split}
\end{equation}



\textbf{SLS.} 
For approaches aiming to achieve better-fused representations from a shared latent space of all modalities, the focus is on enhancing the fusion mechanism of $\mathcal{T}$ through additional losses and modifications to model architectures, denoted as $\zeta_{SLS}$. However, it is important to note that the current shared latent space $\zeta_{SLS}$ methods often do not adequately address the mitigation of modality gaps.

\textbf{DA.} 
Domain adaptation methods address modality gaps by continually adapting the base model $\mathcal{T}$ through specific training with missing modality settings. Considering a possible adaptation $\mathcal{S}_j$ encoder and $\mathcal{C}_j$ classifier derived from $\mathcal{T}$ trained with missing modalities, its objective can be simplified as:
\begin{equation}
    \max_{\mathcal{S}_j,\mathcal{C}_j} 
   \mathbb{E}_{\rmZ_j\sim P(\rmZ_j)}[\ln P(\rmY \mid \mathcal{C}_j(\mathcal{S}_j(\rmX_j)))]. 
\end{equation}


\textbf{Limitations.} 
Albeit their impressive results, most existing methods in the mentioned categories share a common limitation: they often disregard the potential modality gap in $\mathcal{T}$, which can hinder the performance of missing modality tasks. 
Consider a real-valued convex loss $L(\theta)$ where $\theta$ denotes any possible model; the empirical risk~\cite{10433697} introduced by the model $\theta$ for methods without considering the modality gap can be written as
\begin{align}
\label{eq:gap_risk}
    \int  L(Z^*_j, Y_j) dP(\theta) +  \mathbb{E}_{i,j=1}^J D_{KL}(P(Z^*_i)\Vert P(Z^*_j)).
\end{align}
This equation indicates that extra risks would be introduced by the possible modality gaps, thus causing performance degradation. 
Therefore, our approach is based on minimizing the gap across modalities in latent space while maintaining the prediction performance of the latent features.

\subsection{Aligning Medical Multi-modalities}

This section explores an innovative strategy designed to bridge the modality gap, specifically for the base model $\mathcal{T}$. We also demonstrate that our approach can improve many methods across all three categories.

 \textbf{Minimizing modality gap improve generalization ability of the $\mathcal{T}$. }
In the context of domain generalization, minimizing the modality gap can be interpreted as reducing the distributional discrepancy between the source domain $S$ and the target domain $T$ as~\cite{ben2010theory}:
\begin{align}
    \epsilon_T(h) \leq \epsilon_S(h) + d_{H\Delta H}(S, T) + \lambda
\end{align}
where \(\epsilon_T(h)\) is the error rate of the model on the target domain, \(\epsilon_S(h)\) is the error rate of the model on the source domain, \(d_{H\Delta H}(S, T)\) is the distributional discrepancy between the source and target domains, and \(\lambda\) is a term related to the complexity of the hypothesis space.

Reducing the distributional discrepancy between the source and target domains, thereby reducing \(d_{H\Delta H}(S, T)\), which theoretically lowers the error rate on the target domain and thus improves the model's generalization ability. In the modality generalization tasks,  reducing \(d_{H\Delta H}\) of different modalities can also improves the model's generalization ability.

\textbf{How to align to the pre-defined anchor $P_{mix}$. }
Assuming that there is a feasible $P_{mix}$ in the latent space which preserves the prediction performance, i.e., the minimization of $ \int  L(Z^*_j, Y_j) dP(\theta)$ in $Eq.~\ref{eq:gap_risk}$  is guaranteed, 
minimization of $Eq.~\ref{eq:gap_risk}$ 
can be changed as:
\begin{equation}
   \min
   \KL(P( \rmZ_j^*) \Vert P_{mix}) \!
    + \! H_c( P( \mathcal{C}^*(\rmZ_j^*)); P(\rmY)). 
\end{equation}

To derive an empirical form of minimizing the term $\KL(P( \rmZ_j^*) \Vert P_{mix})$, we first introduce Proposition~\ref{prop:P_mix}:
\begin{proposition}
\label{prop:P_mix}
{For training a multi-modal \textit{teacher} model, it is assumed that the existence of one modality $\rmZ_{i} $ is independent of the other modality $\rmZ_{j}$ where $i \in \{1, ..., J\}, j \in \{1, ..., J\}, i \neq j$. In other words, if one modality is missing or corrupted, other modalities can still exist.}
In this scenario, there exists a probability distribution $P_{mix}$ that can be used as an anchor distribution to align the latent variables $\rmZ^{*}$,
while preserving sufficient information for accurate prediction of the segmentation labels $\rmY$.
\end{proposition}
\begin{proof}
The modality existence independent assumption is derived from the fact that each modality is independent of each other.  
If  $P_{mix}$ preserves sufficient information for accurate prediction of the segmentation labels $\rmY$,
based on the joint and marginal mutual information, we have 
\begin{align}
\label{eq:align}
    & \sum\nolimits_{j=1}^{J} \!I(P_{mix}(\rmZ_j^*); P(\rmZ_j^*))  
    \!\le \!I(P_{mix}(\rmZ^*); P(\rmZ^*)).
\end{align}
$Eq.~\ref{eq:align}$ shows that individually mapping each modality $\rmZ_j^*$ to $P_{mix}$ is a lower bound of mapping all modalities together to $P_{mix}$.
\end{proof}

Proposition~\ref{prop:P_mix} is \textbf{single-letterization} that simplifies the optimization problem over a large-dimensional (i.e., multi-letter) problem. 
Therefore, we individually align representations of each modality to the anchor $P_{mix}$ rather than the whole distribution of all representations from all modalities.
Furthermore, 
the former term of $Eq.~\ref{eq:align}$ can be derived as:
\begin{equation}
\label{eq:why_pmix1}
    \sum_{j=1}^{J} 
    \mathbb{E}_{\rmZ_j^*\sim P(\rmZ_j^*)}
    [\ln P(\rmY \! \mid \!\mathcal{C}^*(\rmZ_j^*)) -  \KL(P( \rmZ_j^*) \Vert P_{mix}) ],  
\end{equation}
while the latter term is reformed as:
\begin{equation}
\label{eq:why_pmix2}
    \mathbb{E}_{\rmZ^*\sim P(\rmZ^*)}
    [\ln P(\rmY \! \mid \!\mathcal{C}^*(\rmZ^*)) \! \! - \!\! \KL(P( \rmZ^*) \Vert P_{mix}) ]. 
\end{equation}
{$Eq.~\ref{eq:why_pmix1}$ presents the Evidence Lower Bound (ELBO) and we have $Eq.~\ref{eq:why_pmix1}\leq Eq.~\ref{eq:why_pmix2}$ which indicates that $Eq.~\ref{eq:why_pmix1}$ is tighter than the latter.}
Thus, instead of 
minimizing the gap between all modalities and $P_{mix}$, the alternative objective for 
$\mathcal{T}$ is derived from $Eq.~\ref{eq:why_pmix1}$ as:
\begin{equation}
 \label{eq:t_obj}
   \min
   \sum_{j=1}^J [\KL(P( \rmZ_j^*) \Vert P_{mix}) \!
   + \! H_c( P( \mathcal{C}^*(\rmZ_j^*)); P(\rmY))]. 
\end{equation}


As shown in $Eq.~\ref{eq:t_obj}$,
the essential point is finding a feasible $P_{mix}$ that anchors all latent features in the space while preserving the prediction ability from the latent features to targets for all downstream adaptations.  
We introduce possible forms of $P_{mix}$ as follows.


\textbf{Possible approximations of pre-defined anchor $P_{mix}$.}
\label{approximations}
{The selection of $P_{mix}$ is critical for MedMAP and may vary across different base model backbones. Thus, instead of manual selection, we empirically determine that the optimal $P_{mix}$ is a weighted mixture of all modalities' latent feature distributions. Specifically,  $P_{mix}^* \triangleq \sum\nolimits_{j=1}^J w_j P(\rmZ_j)$ where $w_j$ is the associated weight of each modality. 
To validate $P_{mix}^*$, we compare it to two other possible empirical forms of $P_{mix}$:
1) Intuitive selection $P_{mix}^{k}$ where a latent distribution of one modality is selected from $\{P(Z^*_j)\}_{j=1}^J$ (i.e., $P_{mix}^{k} \triangleq P(Z^*_{j=k})$ where $k\in J$).
2) Assume that $P_{mix}$ follows a normal distribution denoted as $P_{N}$.
Extensive experiments in Section~\ref{sec:exp} demonstrate that using $P_{mix}^*$ consistently yields improvements across different base model backbones under various missing modality settings. validating $P_{mix}^*$'s superiority in comparison to $P_{mix}^{k} $ and $P_{N}$.}

\section{Methodology}
\label{methodology}
\subsection{Alignment Paradigm}
\label{sec:ap}
MedMAP consists of a feature encoding pipeline and an anchor that latent space of the features will be aligned to.
The encoder is a simple convolution that aims to map all the modality features into a shared latent space, so that alignment can be applied. With the given $P_{mix}$, we align each modality latent features $\rmZ^*_{j, j\in\{1,..,J\}}$ to $P_{mix}$.
By converting multi-modal features into a pre-defined distribution, we aim to standardize the scale and distributional properties of data originating from heterogeneous sources, thereby enabling cross-modal learning of distinctive features.
In this article, two kinds of $P_{mix}$ are proposed: $P^{k}_{mix}$ and $P^{*}_{mix}$.
We denote $B$, $J$, the batch size, and the number of modalities, respectively.

\subsubsection{Aligning to $P^{k}_{mix}$}
To derive $P^{k}_{mix}$, we first train the backbones using each modality independently and then evaluate the trained models over all modalities. We select the modality designated as the $k^{th}$
modality, which is identified when the models demonstrate the superior performance when evaluated across $j$ modalities.
Following this, we define $P_{mix}^{k} \triangleq P(\rmZ^*_{j=k})$. 
For aligning other modalities to the selected $k^{th}$ modality,
we minimize:
\begin{align}
\mathcal{L}(\mathcal{T}) = \frac{1}{BJ} \sum\nolimits_{b=1}^B \sum\nolimits_{j=1}^J \left\|(z_j - z_k)\right\|^2,
\end{align}
where $z$ are features of different modalities, $z_j\sim\rmZ_j$ are features that need to be aligned to $x_k\sim\rmZ_k$ from the  $k^{th}$ modality; 

\subsubsection{Aligning to $P^{*}_{mix}$}
\label{sec:pxin}
In the quest to derive a more conducive latent space for integrating all modalities, we have advanced an innovative methodology termed Adaptive Alignment. 
This approach will transcend the basic alignment method that confines the latent space to a specific modality ($P^{k}_{mix}$).
Adaptive Alignment operates under the presumption that an optimal latent space, termed as $P^{*}_{mix}$, for a prior modality can serve as a foundational anchor. Following this, we define $P_{mix}^*\triangleq\sum\nolimits_{j=1}^J w_j P(\rmZ_j)$, where $w_j\in w =\{w\}_{i=i}^J$ are learable weights.
Then we minimize:
\begin{equation}
      \mathcal{L}(\mathcal{T}, w) = \frac{1}{BJ} \sum_{b=1}^B\sum_{j=1}^J  \left\|z_j - \sum\nolimits_{j=1}^J  w_j z_j\right\|^2,
\end{equation}
where 
$x_j\sim\rmZ_j$ are features that need to be aligned, 
and $\sum\nolimits_{j=1}^J  w_j z_j$ is the adaptive anchor learned from all the features. 

Unlike $P^{k}_{mix}$, $P^{*}_{mix}$ does not rely on one specific pre-defined distribution. Such characteristic enables $P^{*}_{mix}$ to more effectively adapt to various multi-modality datasets, requiring less prior knowledge.

\begin{figure}
    \centering
\begin{subfigure}[b]{0.45\textwidth}
    \includegraphics[width=\textwidth]{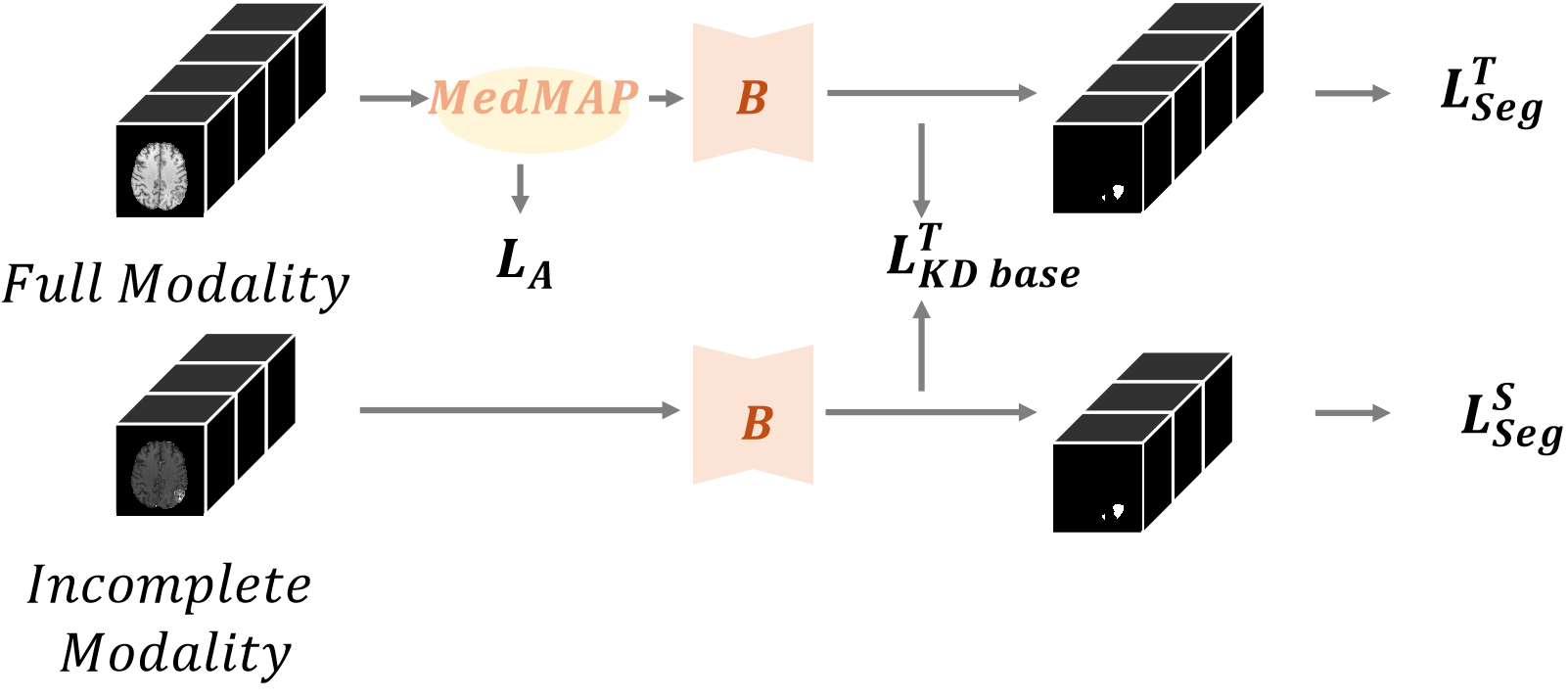}
    \caption{Knowledge Distillation (KD)}
    \label{sls_kd}
\end{subfigure}
\hfill
\hfill
\begin{subfigure}[b]{0.45\textwidth}
    \includegraphics[width=\textwidth]{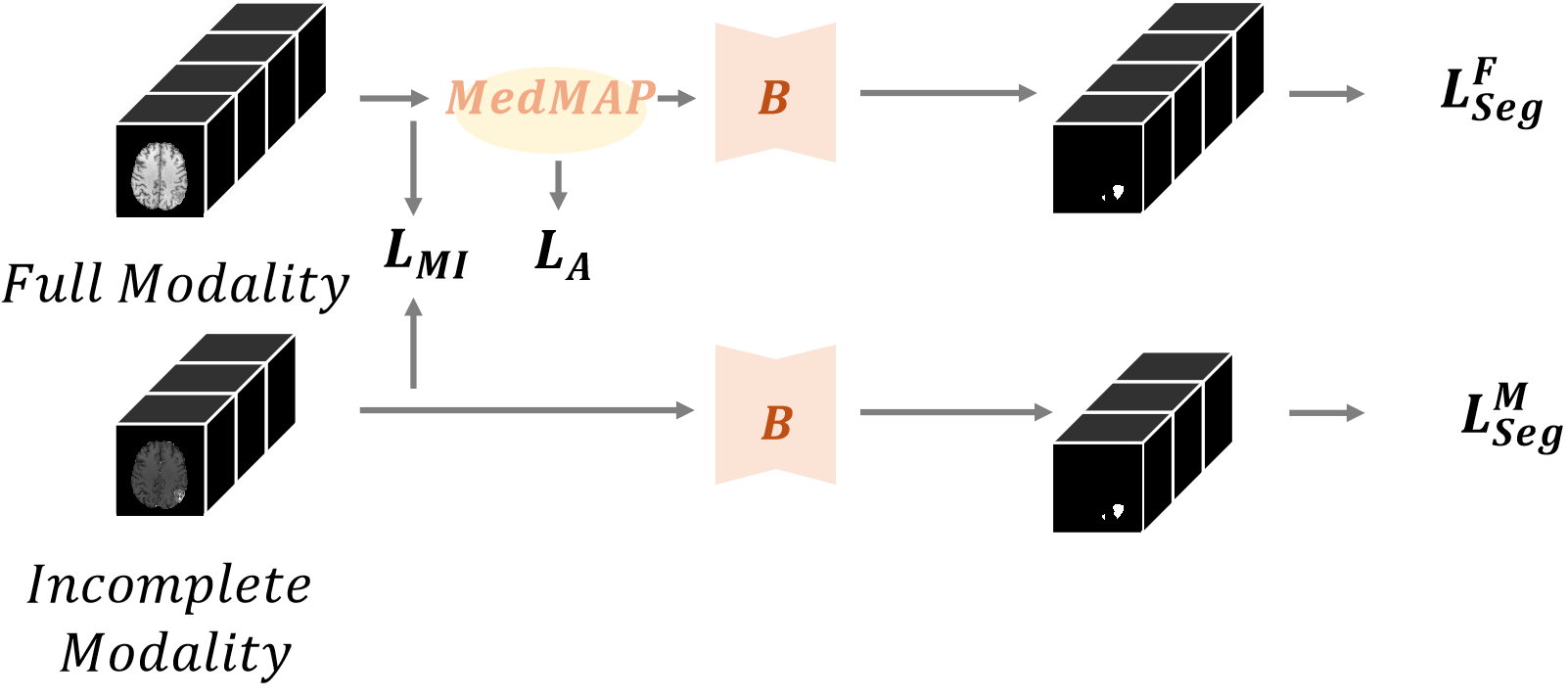}
    \caption{Domain Adaption (DA)}
    \label{sls_da}
\end{subfigure}
\hfill
\begin{subfigure}[b]{0.45\textwidth}
    \includegraphics[width=\textwidth]{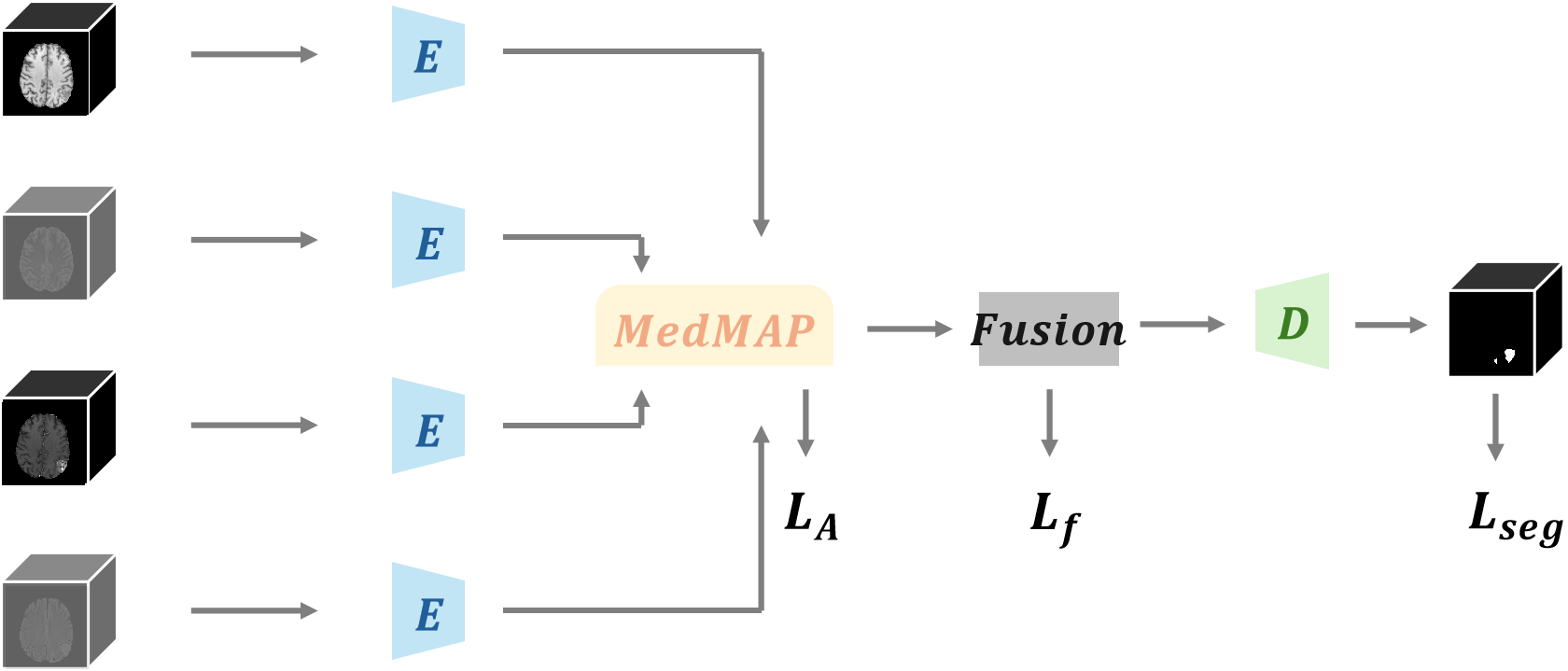}
    \caption{Shared Latent Space (SLS)}
    \label{sls_are}
\end{subfigure}
    \label{fig:baseline_are}
     \caption{Baseline architectures: a) KD, b) DA, c) SLS. \textit{B} is the baseline model structure which represents 3D U-Net according to the baselines. \textit{E}, fusion, and \textit{D} are encoder, fusion, and decoder modules that compose SLS baseline architecture.
     }
\end{figure}

\subsection{Backbones}
\label{sec:backbones}
\subsubsection{KD}
As shown in Fig.~\ref{sls_kd}, in the KD methods, there are two branches: the teacher, and the student. Teacher models are trained with a full-modality dataset. Student models are trained with an incomplete-modality dataset and soft labels from the teacher model. 

For training the teacher, each modality is encoded into a shared latent space and then aligned to the pre-defined $P_{mix}$. Specifically, as shown in $Eq.~\ref{eq:align}$, each modality representation is individually aligned to $P_{mix}$. The aligned features then become the input to the original teacher backbone.
Therefore, the objective of training the teacher will be:
\begin{equation}
    \mathcal{L}_{\text{KD base}}=  \mathcal{L}_{seg}^\mathcal{
    T} +  \alpha\mathcal{L}_{\mathcal{A}},
\end{equation}
where $\mathcal{L}_{seg}^\mathcal{T}$ is the segmentation loss of teacher and $\mathcal{L}_{\mathcal{A}}$ is the alignment loss. $\alpha$ is a parameter which is empirically set to 0.125 in this work.
Given that students harness the distilled knowledge from their teachers during training, their performance serves as a valuable indicator of teacher effectiveness. Therefore, we employ the student performance as a means to validate the feasibility of MedMAP for KD.

Specifically, the teacher model remains frozen, which preserves the integrity of the pre-learned representations and enables us to concentrate on assessing how effectively the student model can emulate the fixed knowledge rather than updating the teachers during student training. 
Consequently, the learning process of the student model is supported by the ground truth and the teacher’s knowledge with the overall objective $\mathcal{L}$ as:
\begin{equation}
    \mathcal{L}_{KD}=  \mathcal{L}_{seg}^\mathcal{
    S} +  \mathcal{L}_{\text{KD base}}^{\mathcal{T}},
\end{equation}
where $\mathcal{L}_{seg}^\mathcal{S}$ is the hybrid segmentation loss supervised by the hard labels of the student, and $\mathcal{L}^{\mathcal{T}}$ denotes the loss that receives supervision from the teacher.

\subsubsection{SLS}
As shown in Fig.~\ref{sls_are}, SLS models always have a modality-specific encoder, a modality fusion encoder, and a decoder~\cite{zhang2022mmformer}. In the modality-specific encoder,  information within a specific modality will be extracted. The modality-fusion encoder will be used to create modality-invariant features with global semantics corresponding to the tumor region. A decoder is used to recover the information. 
We will align the well-extracted feature before the modality-specific encoder. 
Therefore, the total objective of SLS will be: 
\begin{equation}
    \mathcal{L}_{SLS}=  \mathcal{L}_{seg} +  \mathcal{L}_{f} + \alpha\mathcal{L}_{\mathcal{A}},
\end{equation}
where $\mathcal{L}_{seg}$ is the segmentation loss, $\mathcal{L}_{f}$ is the possible fusion loss which has been proposed by the backbones; $\mathcal{L}_{\mathcal{A}}$ is the alignment loss, and $\alpha$ is a parameter which is empirically set to 0.125 in this paper.

\subsubsection{DA}
As shown in Fig.~\ref{sls_da}, similar to KD, DA also has two branches: the full modality branch and the incomplete modality branch. It distills semantic knowledge by minimizing the Kullback-Leibler (KL) between the two branches.
Meanwhile, besides the KL loss, it considers the domain gaps between the full modality model and the incomplete modality model. In ACN~\cite{wang2021acn}, it optimizes the hierarchical mutual information by minimizing the mutual information loss ($\mathcal{L}_{\mathcal{MI}}$) between these two paths.
We argue that we need not only consider the domain gap between the two branches, but also emphasize the domain gap among different modalities. Therefore, we add the MedMAP at the same place with  $\mathcal{L}_{\mathcal{MI}}$. 
It is noted that different from KD, DA is a co-training network. The full-modality model will not be frozen when training. Therefore, the whole objective will be:
\begin{equation}
    \mathcal{L}_{DA}=  \mathcal{L}_{seg}^\mathcal{
    F} +  \mathcal{L}_{seg}^\mathcal{
    M} + \mathcal{L}_{\mathcal{MI}} +   \alpha\mathcal{L}_{\mathcal{A}},
\end{equation}
where $\mathcal{L}_{seg}^\mathcal{
    F}$ and $\mathcal{L}_{seg}^\mathcal{
    M}$ are the segmentation loss of the full modality model and incomplete model respectively and $\mathcal{L}_{\mathcal{A}}$ is the alignment loss.
    $\alpha$ is a parameter which is also set to 0.125 empirically in this paper.

\section{Experiment Configurations}
\label{sec:exp}

\begin{table}[bp]
\caption{Comparison of average Dice (\%) on BraTS2018 and BraTS2020. w.o. is without, w. is with, $\Delta$ is an improvement.}
\centering
\resizebox{\linewidth}{!}{
\begin{tabular}{c|c|ccc}
\toprule
Dataset & Models      & PMKL~\cite{chen2021learning} (KD)   & mmFormer~\cite{zhang2022mmformer} (SLS) & ACN (DA)~\cite{wang2021acn} \\
\bottomrule
\bottomrule
\multirow{2}{*}{BraTS2018} & w.o. MedMAP     &  65.44      &   70.62       &    70.54 \\
& w. MedMAP ($\Delta$) & 69.12  (\red{3.68} \textcolor{red}{$\uparrow$})       &     78.29 (\red{7.67} \textcolor{red}{$\uparrow$})    &   72.84 (\red{2.30} \textcolor{red}{$\uparrow$})\\
\bottomrule
\multirow{2}{*}{BraTS2020} & w.o. MedMAP    &   72.44   &          72.94       &   74.38  \\
& w. MedMAP ($\Delta$) &   73.75 (\red{1.31} \textcolor{red}{$\uparrow$})   &     74.76 (\red{6.76} \textcolor{red}{$\uparrow$}) &     76.41 (\red{0.38} \textcolor{red}{$\uparrow$})      \\
\bottomrule
\end{tabular}
}
\label{tab:avergaImp}
\end{table}

\subsection{Datasets}
The BraTS datasets in 2018 and 2020~\cite{menze2014multimodal,bakas2017advancing}, consisting of 285 and 369 subjects respectively, are employed for evaluation. 
Four individual modalities including T1, T1ce, T2, and Flair construct a specific subject, displaying brain tumor sub-regions \textit{i.e.} enhancing tumor (ET), peritumoral edema (ED), and the necrotic and non-enhancing tumor core (NCR/NET).
Specifically, each modality captures different properties of brain tumor sub-regions: GD-enhancing tumor (ET), peritumoral edema (ED), and the necrotic and non-enhancing tumor core (NCR/NET), nesting into three key segmentation targets, \textit{i.e.}, whole tumor (WT, ET+ED+NCR/NET), tumor core (TC, ET+NCR/NET), and ET. 
For both datasets, we use Dice Score ($\%$)~\cite{dice1945measures} to evaluate the performance:  higher Dice indicates better segmentation performance. 

\subsection{Implementation}
For comparison fairness, we implant the proposed alignment paradigm to retrain each employed backbones~\cite{chen2021learning,zhang2022mmformer,wang2021acn} following identical respective experimental settings except the dataset configuration and the batch size, given by each author's released codes.
The BraTS 2018 dataset is split into training, validation, and testing sets following Proto-KD~\cite{wang2023prototype}, while the other BraTS 2020 dataset is used with three-fold cross-validations. 
In contrast to its official version, the batch size of mmFormer~\cite{zhang2022mmformer} is set to 1 and the batch size of PMKL~\cite{chen2021learning} is set to 4 while the others remain the same. Models are implemented on one Nvidia GeForce RTX 3090Ti GPU with Pytorch.


\begin{table*}[]
\caption{Comparison of dice scores when different modalities are missing on BraTS2018 and BraTS2020. $\bullet$ represents present modalities. $\circ$ displays missing modalities, while N denotes the number of them.
Statistics in rows with + display achieved performance of the proposed paradigm, while those in rows of $\Delta$ evaluate performance difference brought by the proposed MedMAP paradigms.
Statistics in columns titled by Avg. present average dice scores, while those in Total Avg. columns display average dice scores across different missing modality scenarios.}
\centering
\tiny
\resizebox{\linewidth}{!}{%
\begin{tabular}{c|c|cccc|c|c|c|cccc|c|c|c}
\toprule
\multicolumn{2}{c}{} &\multicolumn{7}{|c}{\textbf{BraTS2018}}&\multicolumn{7}{|c}{\textbf{BraTS2020}}\\
\midrule
\multirow{4}{*}{Type} & Flair& $\circ$& $\circ$& $\circ$& $\bullet$& \multirow{4}{*}{\makecell[c]{Avg. \\ N=2}}& \multirow{4}{*}{\makecell[c]{Avg. \\ N=1}}& \multirow{4}{*}{\makecell[c]{Total \\ Avg.}} & $\circ$&$\circ$&$\circ$&$\bullet$&\multirow{4}{*}{\makecell[c]{Avg. \\ N=2}}&\multirow{4}{*}{\makecell[c]{Avg. \\ N=1}}& \multirow{4}{*}{\makecell[c]{Total \\ Avg.}} \\
& T1& $\circ$& $\circ$& $\bullet$&$\circ$& & & & $\circ$&$\circ$&$\bullet$&$\circ$&&& \\
& T1ce& $\circ$& $\bullet$& $\circ$&$\circ$& & & & $\circ$&$\bullet$&$\circ$&$\circ$&&& \\
& T2 &$\bullet$& $\circ$& $\circ$&$\circ$& & & & $\bullet$&$\circ$&$\circ$&$\circ$&&&\\
\midrule
\multirow{9}{*}{WT}& PMKL&{81.00}&70.50&73.31&84.11&76.72&78.10&77.32&82.79&75.76&76.92&86.19&84.17&87.67&83.75\\
& +
MedMAP&83.77&75.97&72.04&85.70&81.35&83.85&\textbf{81.06}&84.23&77.03&76.85&86.79&85.44&88.06&\textbf{84.74}\\
& $\Delta$&\red{2.77} & \red{5.47} & \blue{-1.27} & \red{1.59} &\red{4.63} &\red{5.75}& \cellcolor{mygray}\red{3.73}&\red{1.44} &\red{1.27} & \blue{-0.07}& \red{0.60}&\red{1.27}&\red{0.89}&\cellcolor{mygray}\red{0.99}\\
\cline{2-16}
& mmFormer &84.09&72.85&73.37&85.60&85.26&87.81&83.67&85.78&76.46&76.35&87.80&87.81&89.15&85.79\\
& + MedMAP &85.32 & 85.50 & 85.71 & 84.51 &86.61 &86.97&\textbf{86.26}&87.46&87.74&87.51& 87.77& 87.86&87.96&\textbf{87.79}\\
&$\Delta$& \red{1.23} & \red{12.65} & \red{12.34} & \blue{-1.09} &\red{1.35} & \blue{-0.83}& \cellcolor{mygray}\red{2.59}& \red{1.68}& \red{11.28}& \red{11.16}& \blue{-0.03}&\red{0.07}&\blue{-1.79}&\cellcolor{mygray}\red{2.00}\\
\cline{2-16}
& ACN&84.18&77.31&73.63&84.77&82.68&85.04&82.38&86.94&78.15&78.58&87.33&86.99&88.98&85.74\\
& + MedMAP&86.68&78.09&79.31&87.66&83.98&85.25&\textbf{83.95}&86.57&78.91&78.56&87.54&86.98&88.77&\textbf{85.98}\\
&$\Delta$& \red{2.50} & \red{0.78} & \red{5.68} & \red{2.89} &\red{1.30} &\red{0.21}& \cellcolor{mygray}\red{1.57}& \blue{-0.37}& \red{0.76}& \blue{-0.02}& \red{0.21}&\blue{-0.01}&\blue{-0.21}&\cellcolor{mygray}\red{0.24}\\
\midrule
\multirow{9}{*}{TC}&PMKL&67.92&76.92&64.26&62.21&66.00&68.96&66.60&71.50&82.37&69.39&67.10&75.27&79.49&74.85\\
&+ MedMAP&69.91&80.35&68.39&68.44&68.14&74.24&\textbf{69.83}&71.78&82.48&71.63&69.98&77.05&81.05&\textbf{76.13}\\
&$\Delta$& \red{1.99}&\red{3.43} &\red{4.13} & \red{6.23} & \red{2.14}&\red{5.28}& \cellcolor{mygray}\red{3.23}&\red{0.38} & \red{0.11}& \red{2.24}& \red{1.22}&\red{1.78}&\red{1.56}&\cellcolor{mygray}\red{1.18}\\
\cline{2-16}
& mmFormer &67.80&77.32&64.56&64.08&75.01&78.99&73.04&71.03&81.06&68.77&72.14&79.12&82.69&73.23\\
& + MedMAP & 77.76 & 77.34 & 76.63 & 76.29 &78.46 &78.72&\textbf{78.02}&82.35&81.92& 71.49&81.99&82.33&82.89&\textbf{76.22}\\
&$\Delta$& \red{9.96}& \red{0.02}& \red{12.07}& \red{12.21}&\red{3.44} &\blue{-0.27}& \cellcolor{mygray}\red{4.97}& \red{11.32}& \red{0.86}& \red{2.72}& \red{9.85}&\red{3.21}&\red{0.20}&\cellcolor{mygray}\red{2.99}\\
\cline{2-16}
& ACN&72.08&66.17&59.74&73.26&71.96&71.17&72.57&75.21&71.95&67.21&76.43&77.34&80.82&76.64\\
& + MedMAP &74.41&70.42&67.13&76.74&73.79&76.90&\textbf{73.49}&75.76&70.46&67.26&76.92&76.88&81.03&\textbf{76.96}\\
&$\Delta$& \red{2.33}& \red{4.25}& \red{7.39}& \red{3.48}& \red{1.83}&\red{0.15}& \cellcolor{mygray}\red{2.32}&\red{0.55}& \blue{-0.49}& \red{0.05}& \red{0.49}&\blue{-0.46}&\red{0.22}&\cellcolor{mygray}\red{0.32}\\
\midrule
\multirow{9}{*}{ET}&PMKL&47.09&75.54&41.37&41.35&52.52&60.70&52.39&44.37&74.19&44.52&44.57&61.74&70.91&58.73\\
&+ MedMAP&45.17&76.44&47.66&43.57&56.86&69.27&\textbf{56.47}&48.17&76.64&44.82&49.02&62.68&70.44&\textbf{60.38}\\
&$\Delta$& \blue{-1.92}& \red{0.90}& \red{6.29}& \red{2.22}& \red{4.34}&\red{8.57}& \cellcolor{mygray}\red{4.09}&\red{3.80}&\red{2.45} & \red{0.30}& \red{4.45}&\red{0.94}&\blue{-0.47}&\cellcolor{mygray}\red{1.65}\\
\cline{2-16}
& mmFormer &40.08&72.19&38.89&37.23&58.05&68.44&55.15&46.30&75.42&43.80&44.26&63.04&71.46&59.81\\
&+ MedMAP&69.02&70.55&70.77&70.23&70.72&70.93&\textbf{70.58}&75.00&74.83&74.67&74.30&75.59&75.60&\textbf{75.10}\\
&$\Delta$& \red{28.94}& \blue{-1.64}& \red{31.88}& \red{33.00}&\red{15.43}&\red{2.49}& \cellcolor{mygray}\red{11.72}&\red{28.70}& \blue{-0.59}& \red{30.87}& \red{30.04}&\red{12.55}&\red{4.14}&\cellcolor{mygray}\red{15.29}\\
\cline{2-16}
& ACN&47.79&75.60&43.25&42.21&58.79&66.55&56.66&49.27&77.55&46.65&46.70&63.38&71.19&60.77\\
& $+$ MedMAP &47.93&78.03&50.64&46.61&65.17&67.07&\textbf{61.09}&48.91&78.27&48.34&47.94&64.79&71.80&\textbf{61.35}\\
&$\Delta$& \red{0.14}& \red{2.43}& \red{7.39}& \red{4.40}& \red{6.39}&\red{0.52}& \cellcolor{mygray}\red{4.43}&\blue{-0.26}& \red{0.62}& \red{1.69}& \red{1.24}&\red{1.41}&\red{0.61}&\cellcolor{mygray}\red{0.58}\\
\bottomrule
\end{tabular}%
}
\label{tab:MissingOn18}
\end{table*}

\section{Experiment Results}
As discussed in Sec.~\ref{sec:reviewscope}, recent approaches to predict annotations in the \underline{M}issing \underline{M}odality scenarios (MM) mainly include Knowledge Distillation (\textit{e.g.} PMKL~\cite{chen2021learning}), Shared Latent Space (\textit{e.g.} mmFormer~\cite{zhang2022mmformer}), and Domain Adaption (\textit{e.g.} ACN~\cite{wang2021acn}). 
The proposed Alignment Paradigm will be implanted into the
aforementioned recent SOTAs to evaluate the effectiveness.
Specifically, Section~\ref{sec:QuantitativeComparison} will present statistical improvements given by the proposed MedMAP, measured by averaged dice as well as improvement (marked in red) and decline (in blue) in Table~\ref{tab:avergaImp}. 
Details can also be found across various MM on both datasets in Table~\ref{tab:MissingOn18}.
Sec.~\ref{sec:QualitativeComparison} visualizes examples with and without $P_{mix}$.
In Sec.~\ref{sec:apcomponent}, we compare several possible structures in key components of the proposed MedMAP, which guide us to find the most suitable encoder and anchor.
In Sec.~\ref{sec:Hyperparamters}, the selection of hyperparameters will be introduced.



\subsection{Quantitative Comparisons}
\label{sec:QuantitativeComparison}
\subsubsection{Averaged results across prediction classes and MM} 
As seen from Table~\ref{tab:avergaImp}, we can conclude that the proposed MedMAP effectively promotes the prediction performance. Specifically, on the BraTS2018 dataset, it improves $3.68\%$, $7.67\%$, and $2.30\%$ of the Dice Scores on PMKL, mmFormer, and ACN respectively. Improvements are $1.31\%$, $6.76\%$, and $0.38\%$ of the Dice Scores on PMKL, mmFormer, and ACN respectively on the BraTS2020 dataset.

\subsubsection{Performance improvement on different prediction classes}
In both Table~\ref{tab:MissingOn18}, the right-most columns of each dataset (Total Avg.) indicate the average performance across different MM. 
From them, we can identify when the proposed MedMAP is absent, the mmFormer achieves $83.67\%$, $73.04\%$, and $55.15\%$, respectively, on the BraTS2018 dataset, and $85.79\%$, $73.23\%$, and $59.81\%$ on BraTS2020 when predicting WT, TC, and ET. 
Notably, on both datasets, predictions obtain lower dice scores on TC and ET than on the WT.
The presence of the proposed MedMAP significantly promotes the degraded predictions of TC and ET by $4.97\%$ and $11.72\%$ respectively on BraTS2018, and $2.99\%$ and $15.29\%$ on BraTS2020. 
Moreover, such achievement can also be identified in the mmFormer predictions on the WT by $2.59\%$ and $2.00\%$ respectively on both involved datasets. 

Furthermore, the proposed MedMAP also boosts the performance of the other two backbones. Specifically, it improves PMKL by $3.73\%$, $3.23\%$, and $4.09\%$ respectively on WT, TC, and ET of BraTS2018, and $0.99\%$, $1.18\%$, and $1.65\%$ on BraTS2020. It improves ACN by $1.57\%$, $2.32\%$, and $4.43\%$ on BraTS2018, as well as $0.24\%$, $0.32\%$, and $0.58\%$ on BraTS2020. 



\subsubsection{MedMAP improvement in different MM}
Since there are four modalities in each of the BraTS2018 and BraTS2020 datasets, there are in total $2^4-1=15$ possible MM. We will evaluate each of them in Table~\ref{tab:MissingOn18}.  Columns in the tables are categorized and split by the number of MM apart from the right-most column: sub-columns from the left to the right contain results when $N$ of the available four modalities are absent, where $N \in [3,2,1]$. When $N=1$, all the detailed results are shown and when $N=2$ and $N=3$, the average results are shown. Although there are some slight decreases when adding our MedMAP, the improvement can be seen in most scenarios.

\subsection{Qualitative Comparison}
\label{sec:QualitativeComparison}
We can observe from Fig.~\ref{fig:visual_all} that 
in general, the performance of the model with $P_{mix}$is superior to that without it. When comparing the first and second rows, as well as the first two columns of the third row (cases with three or two missing modalities), we find that the model without $P_{mix}$ struggles to predict NCR/NET, with instances of both missed detections and false detections. Upon observing the predictions of the first and fourth rows, we observe that the model with $P_{mix}$ demonstrates greater accuracy in predicting ED, with smoother ED edges and the ability to predict details. In conclusion, in cases with $N$ modalities missing, the model with $P_{mix}$ performs consistently better.


\begin{figure*}[ht]
\centering
\hfill 
\begin{subfigure}[b]{0.105\textwidth}
    \includegraphics[width=\textwidth]{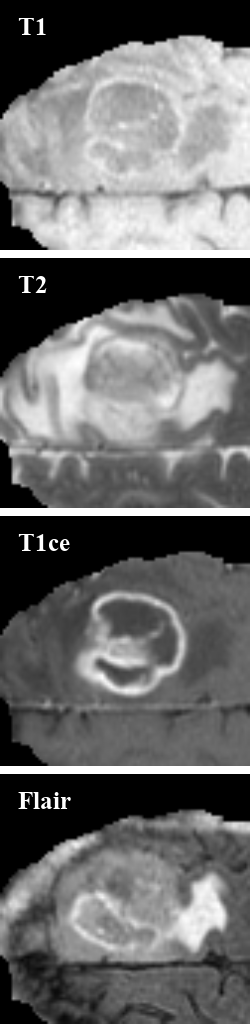}
    \caption{GT}
\end{subfigure}
\hfill 
\begin{subfigure}[b]{0.43\textwidth}
    \includegraphics[width=\textwidth]{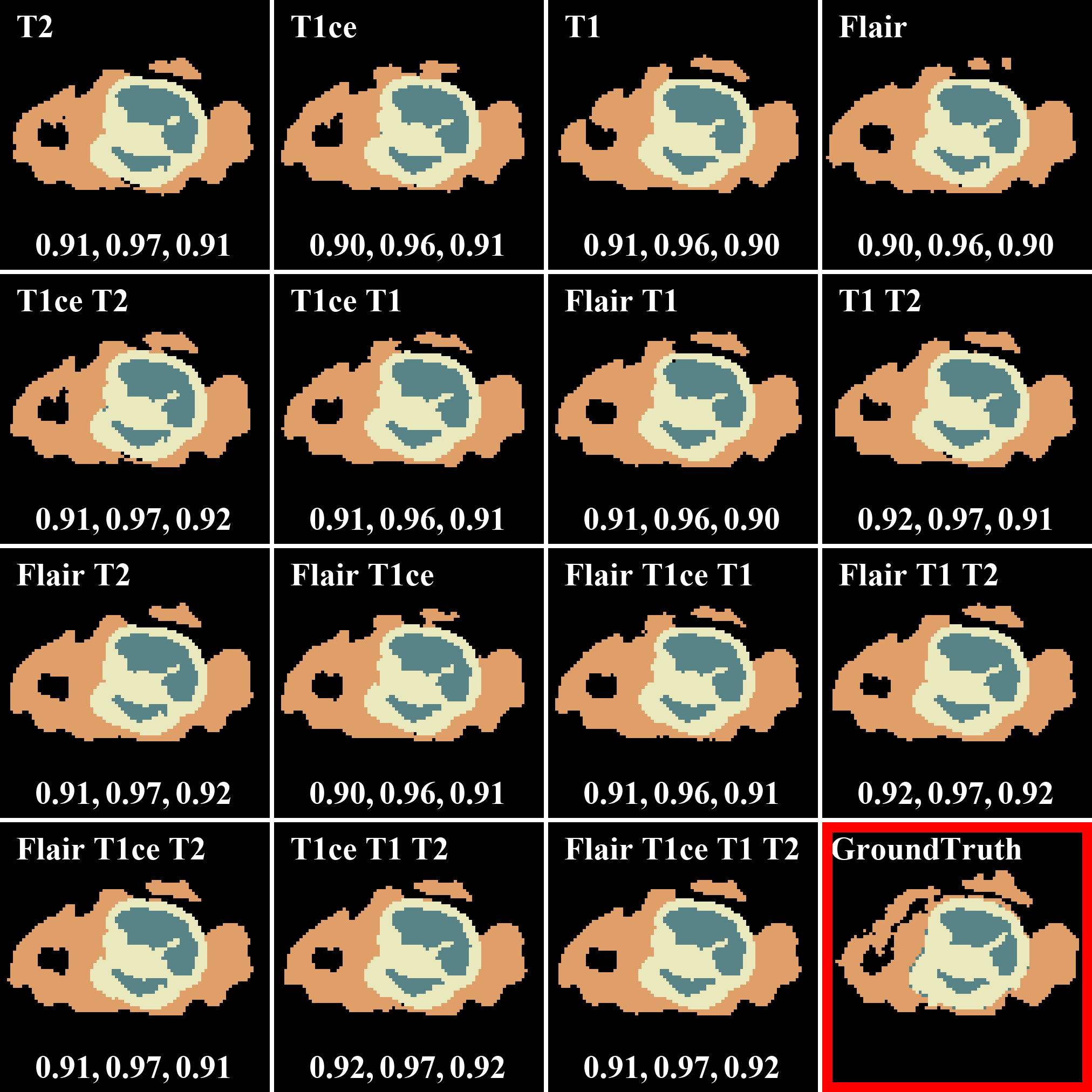}
    \caption{with $P_{mix}$}
\end{subfigure}
\hfill 
\begin{subfigure}[b]{0.43\textwidth}
    \includegraphics[width=\textwidth]{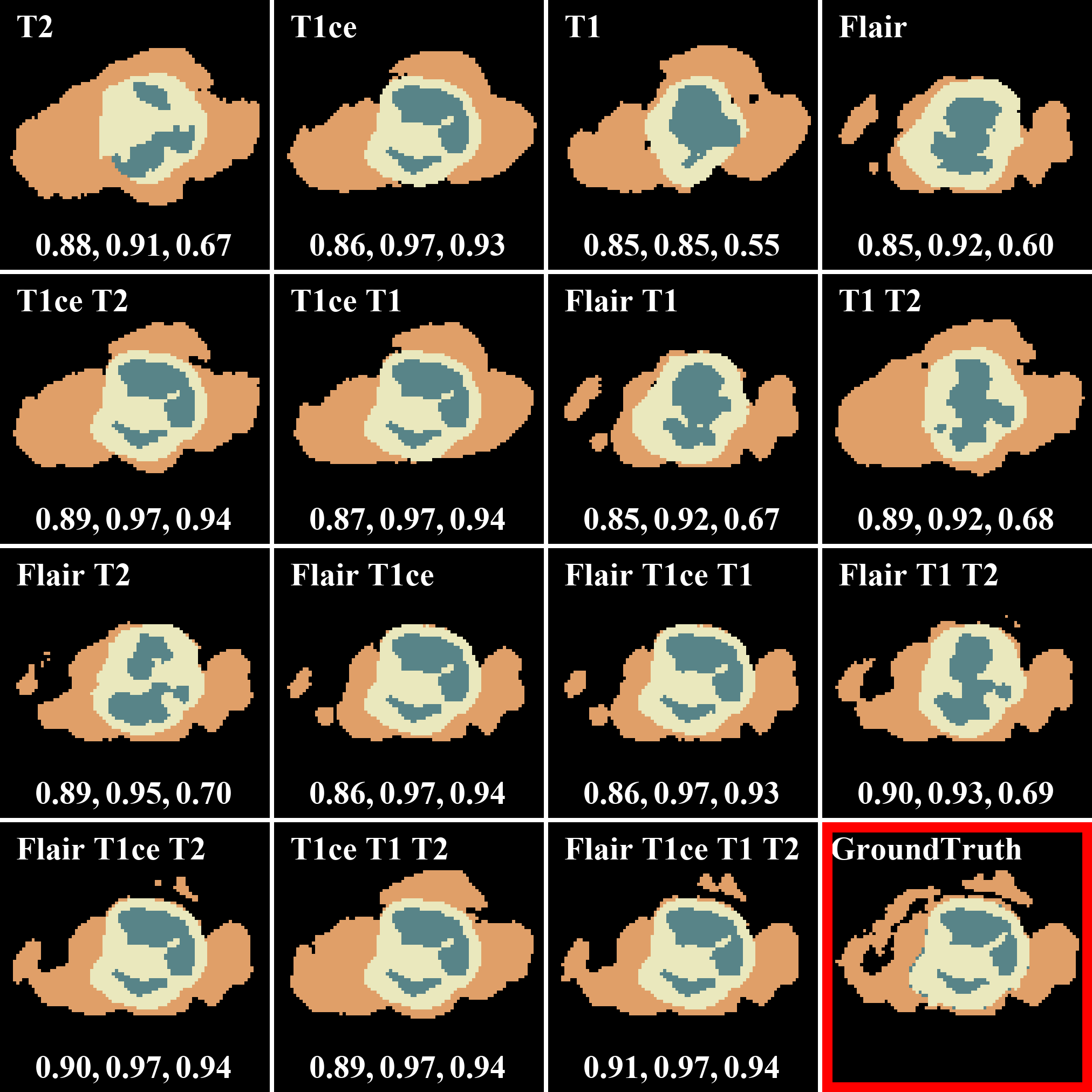}
    \caption{without $P_{mix}$}
\end{subfigure}
\caption{Qualitative segmentation results of mmFormer with $P_{mix}$ and without $P_{mix}$ on BraTS2018 under all missing scenarios. Below each sub-figure is the Dice. The Dice texts from left to right are WT, TC, and ET. Different colors represent different organs: Blue: NCR/NET, Orange: ED, and Yellow: ET. Captions on the upper-left corners indicate present modalities. Bottom-right corner is the groundtruth labels.}
\label{fig:visual_all}
\end{figure*}

\begin{figure*}[h]
\centering

\hfill 
\begin{subfigure}[b]{0.0725\textwidth}
    \includegraphics[width=\textwidth]{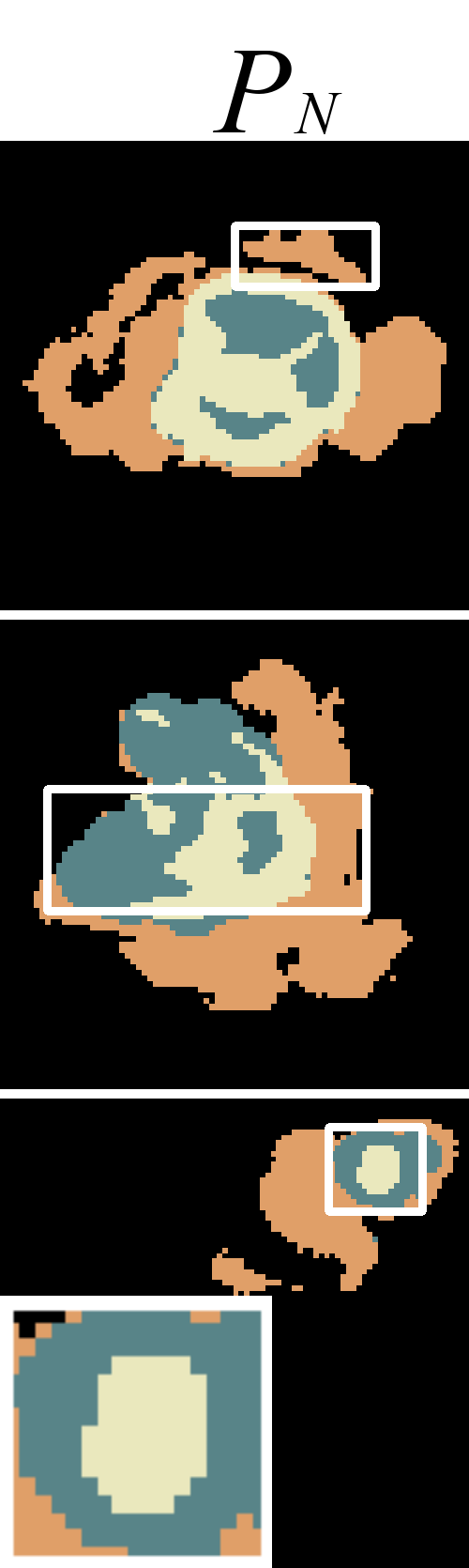}
    \caption{GT}
    \label{fig:withp}
\end{subfigure}
\hfill 
\begin{subfigure}[b]{0.22\textwidth}
    \includegraphics[width=\textwidth]{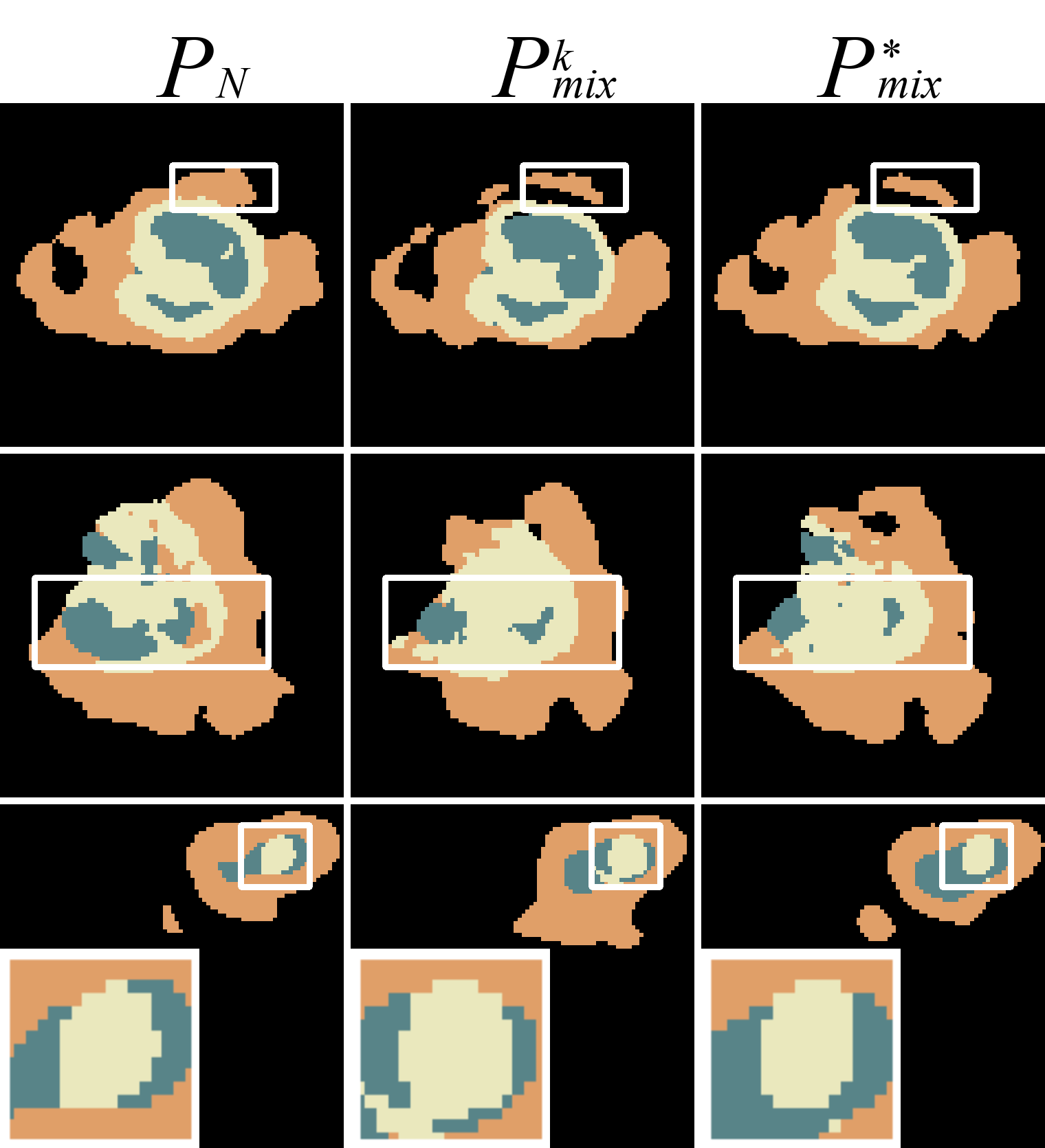}
    \caption{T1}
    \label{fig:withp}
\end{subfigure}
\hfill 
\begin{subfigure}[b]{0.22\textwidth}
    \includegraphics[width=\textwidth]{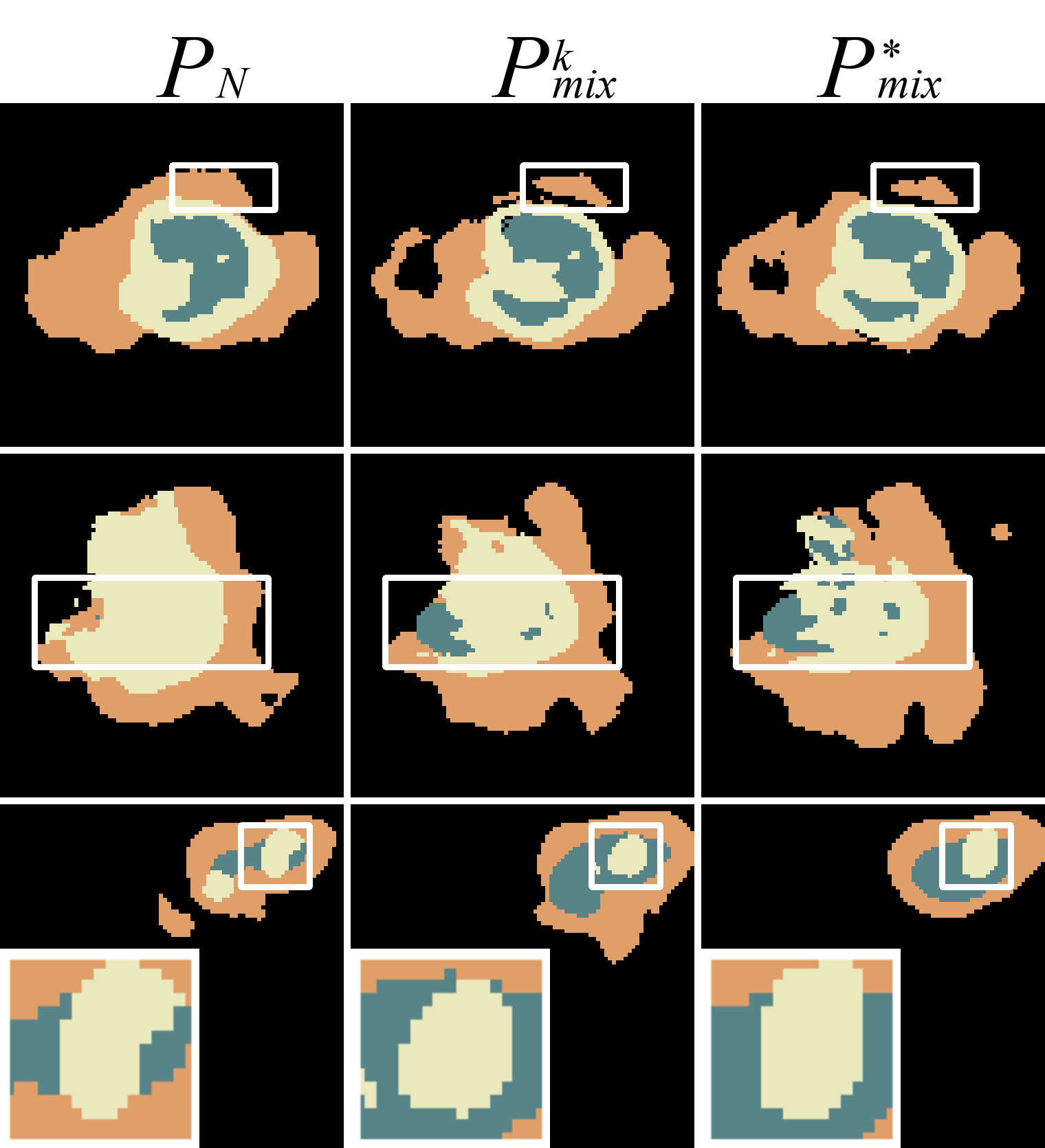}
    \caption{T2}
    \label{fig:withp}
\end{subfigure}
\hfill 
\begin{subfigure}[b]{0.22\textwidth}
    \includegraphics[width=\textwidth]{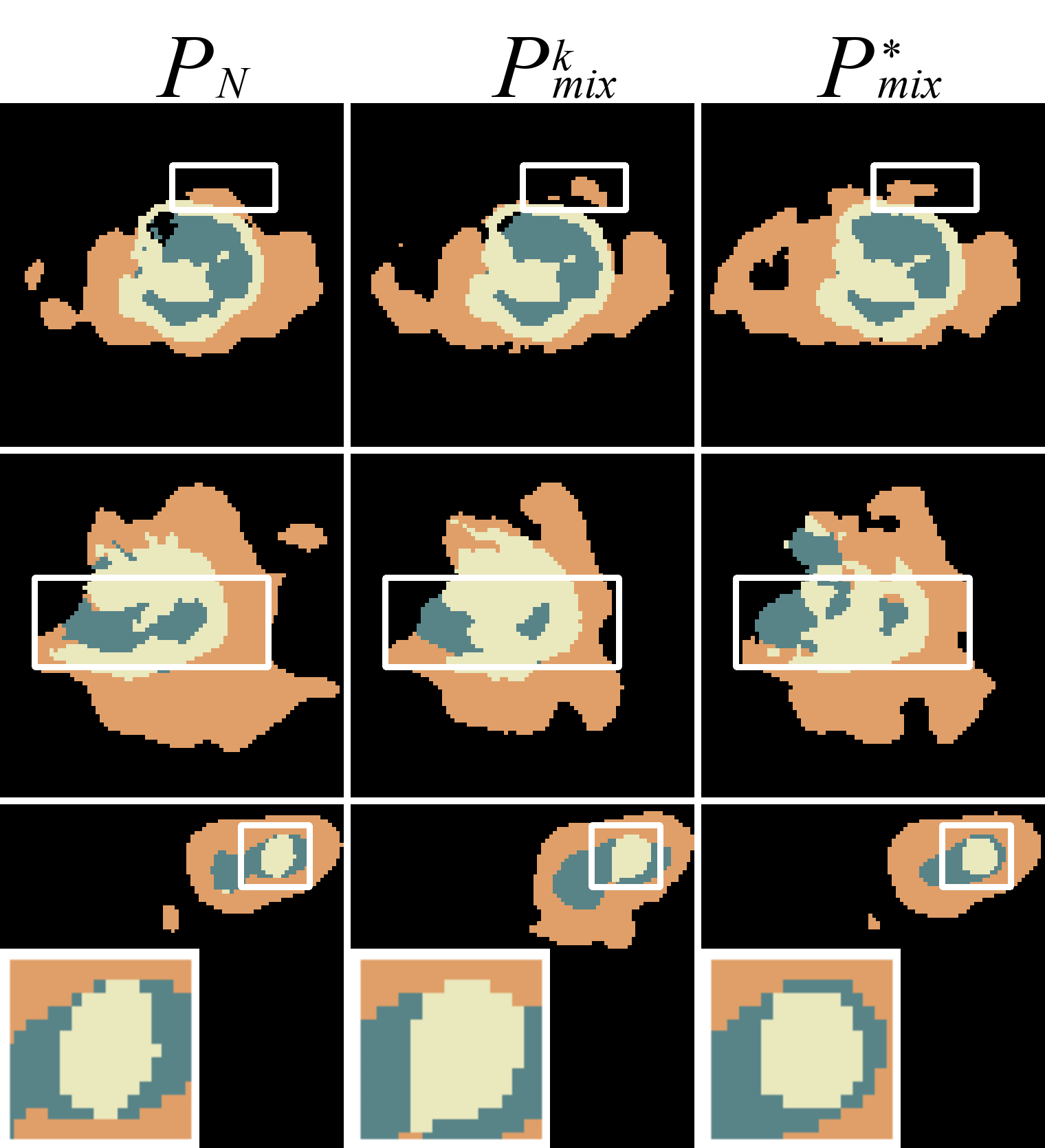}
    \caption{T1ce}
    \label{fig:withp}
\end{subfigure}
\hfill 
\begin{subfigure}[b]{0.22\textwidth}
    \includegraphics[width=\textwidth]{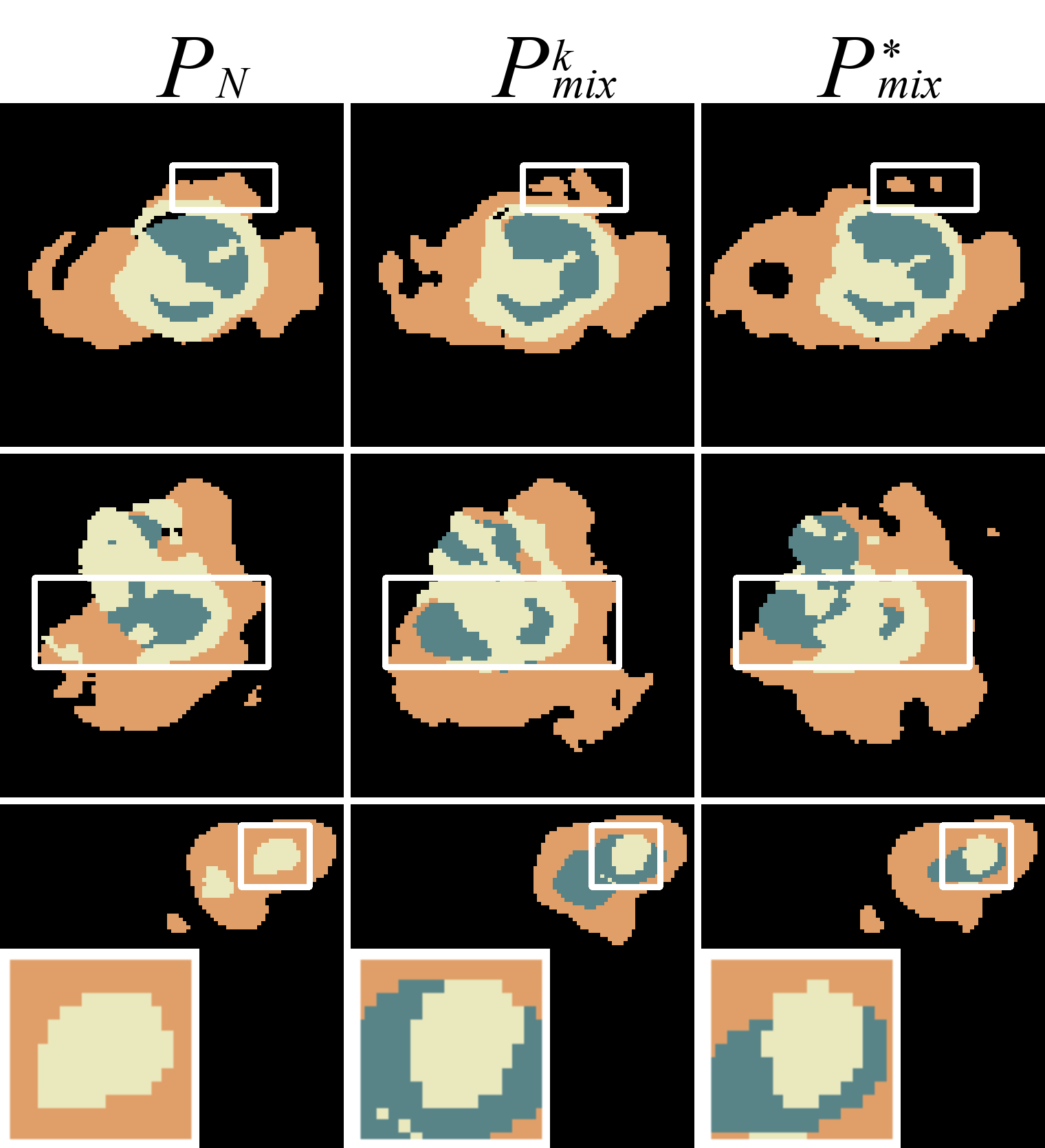}
    \caption{Flair}
    \label{fig:withp}
\end{subfigure}
\label{fig:withoutp}
\caption{Qualitative segmentation results of mmFormer with $P_{N}$, $P^{k}_{mix}$ and $P^{*}_{mix}$ on BraTS2018 where three modalities are missing. }
\label{fig:visual_ablation}
\end{figure*}

\subsection{Comparison of Alignment Paradigm components}
\label{sec:apcomponent}
As part of the proposed MedMAP with an encoder and an aligning anchor, we conduct ablation studies to evaluate the effectiveness of both components. Regarding the encoder, we compare the enhanced and the non-shared encoder. 
For the anchor $P_{mix}$, we compare the  proposed $P_{mix}$ with the standard normal distribution~\cite{dorent2019hetero}. 
Furthermore, we assess two different $P_{mix}$ options namely $P^{k}_{mix}$ and $P^{*}_{mix}$.

\begin{figure}
    \centering
            \includegraphics[width=\linewidth]{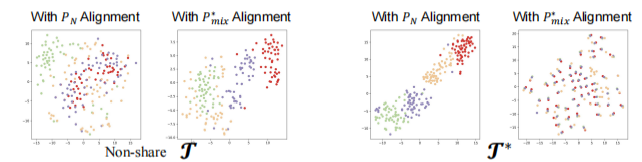}
    \caption{T-SNE visualization of different anchors and encoders. Different colors mean different modalities.}
    \label{fig:ablation_tsne}
\end{figure}

\subsubsection{Encoder}
As discussed in Sec.~\ref{sec:ap}, the proposed alignment paradigm will only be possible when latent features are placed in the same space, by using some encoding architectures.
In this sense, we would like to evaluate the prediction performance based on two different encoder configurations: the non-shared encoder and the enhanced encoder, termed as non-shared $\mathcal{T}$ and $\mathcal{T^*}$,  respectively. 
In non-shared $\mathcal{T}$, an independent encoder will be used to process each modality and channelized individually.
Consequently, with $J$ modalities, there will be $J$ individual encoders.
The total number of channels for the non-shared $\mathcal{T}$ is $J\cdot c$, where $c$ represents the number of channels of the input feature in one modality. 
The number of channels of $\mathcal{T^*}$ is set to be $J\cdot c$ for fair comparison given an equivalent amount of training weights.

To examine closer the potential difference obtained, we evaluate the simplest involved backbone PMKL~\cite{chen2021learning} based on U-Net architecture~\cite{cciccek20163d}, and in the most challenging testing scenario when three over four in BraTS2018 modalities are missing.
As such, the simplest model is sensitive to the change in model structure and the change of the information amount received.
Comparing the results are presented in Table~\ref{tab:teacherResultencoder}, we can find that incorporating non-shared  $\mathcal{T}$ achieves a slight improvement in each modality. Further gains on the average dice scores can be identified by introducing $\mathcal{T^*}$. 
Figure~\ref{fig:ablation_tsne} illustrates t-SNE~\cite{van2008visualizing} embeddings, where we can find that incorporating $\mathcal{T^*}$ is able to narrow gaps between modalities. Therefore, $\mathcal{T^*}$ is chosen as an appropriate encoder.

\begin{table}[t]
\scriptsize
\tiny
\centering
\caption{Comparison results with different types of encoders. Best results are highlighted as \textbf{bold}.}
\setlength{\tabcolsep}{3mm}{
\begin{tabular}{c|cccc|c|c}
\toprule
Present Modality           & WT    & TC    & EC    & Average & Avg. & $\mathcal{T}$   \\ 
\midrule
\midrule
\multirow{1}{*}{T1} &73.31 &64.26 &{41.37} &58.98 & \multirow{4}{*}{65.30} & \multirow{4}{*}{without $\mathcal{T}$}\\
\multirow{1}{*}{T2}    &81.00 &67.92& {47.09} &{65.34} &  \\
\multirow{1}{*}{T1ce}   & 70.50 & 76.92 & 75.54 & 74.32 &\\
\multirow{1}{*}{Flair}   & 84.11& 62.21& 41.35 &62.56 & \\
\midrule
\multirow{1}{*}{T1}  &  72.99&63.28&38.92&58.40 & \multirow{4}{*}{65.72} & \multirow{4}{*}{non-share $\mathcal{T}$}\\
\multirow{1}{*}{T2}  &    81.11 & {68.40}& 45.99& 65.17 &\\
\multirow{1}{*}{T1ce}  &  {74.51}&{78.99}&{73.85}&{75.78}& \\
\multirow{1}{*}{Flair}  &  83.98&{67.10}&39.48&63.52& \\
\midrule
\multirow{1}{*}{T1}   & {75.50}  & {65.98}   & 40.09  &{60.53} &\multirow{4}{*}{\textbf{66.40}} & \multirow{4}{*}{\textbf{$\mathcal{T^*}$}}\\               
\multirow{1}{*}{T2}    &{82.68}  &  67.14  &  {44.82} & 64.88&\\               
\multirow{1}{*}{T1ce}    &74.00 &  78.64  & 72.71  &  {75.12}&\\               
\multirow{1}{*}{Flair}    &{84.74}   &   67.07 & {43.42}  &{65.07}&\\               
\midrule
\end{tabular}}
\label{tab:teacherResultencoder}
\end{table}

\subsubsection{Comparison with standard normal distribution}
We will compare the prediction performance achieved by setting the anchor to $P_N$ which is a standard normal distribution.
In particular, given $J$ modalities, $P_{mix}$ is achieved by minimizing $\mathbb{E}[2\log {1}/{v(\rmZ^*_j)} + {v(\rmZ^*_j)^2 + ({\Bar{\rmZ}^*_j})^2}/{2} - {1}/{2}]$, where 
$\Bar{\rm Z}^*_j$ and $v(\rm Z^*_j)$ denote the mean and variance of ${\rmZ}^*_j$, which are optimized during training.

\begin{table}[t]
\tiny
\centering
\setlength\tabcolsep{2.5pt}
\caption{Comparison of segmentation results in each class and average dice scores with different anchor $P_{mix}$s.}
\label{tab:ablationp}
\begin{tabular}{c|cccc|cccc|cccc|c}
\toprule
\multicolumn{1}{c|}{\multirow{2}{*}{Method}} & \multicolumn{4}{c|}{ with $P_{N}$} & \multicolumn{4}{c|}{ with $P^k_{mix}$}         & \multicolumn{4}{c|}{ with $P^*_{mix}$ } & \multicolumn{1}{c}{\multirow{2}{*}{Modality}}\\
\cmidrule{2-13} 
\multicolumn{1}{c|}{}                       & WT    & TC    & EC    & Avg.     &WT& TC    & EC    & Avg.          & WT     & TC     & EC    & Avg.  &        \\
\midrule
\midrule
Teachers & 84.14         & 77.44      & 74.82      &  79.13 &  86.34   & 79.75 & 76.91 & 81.00 &  86.81& 79.22 & 77.85 & \textbf{81.29} &Full\\
\midrule
PMKL~\cite{chen2017rethinking}                                         &{75.60}&	65.59	&43.31 &  61.50  &75.06	&66.80	&41.43	&61.10  &72.04	&{68.39}	&{47.66}	&\textbf{62.70} & \multirow{3}{*}{T1}\\
mmFormer~\cite{zhang2022mmformer}&83.86&75.70&65.94&75.17&84.83&76.80&68.87&76.83&85.71&76.63&70.77&\textbf{77.70}&\\
ACN~\cite{wang2021acn}&74.52&61.65&45.99&60.72&74.95&66.32&45.82&62.36&79.31&67.13&50.64&\textbf{65.69}&\\
\midrule
PMKL~\cite{chen2017rethinking}                                            &80.46&	69.06	&48.38&	65.97  &82.47	&69.56	&45.78	&65.94  &{83.77}	&{69.91}	&45.17	&\textbf{66.28} &  \multirow{3}{*}{T2}\\
mmFormer~\cite{zhang2022mmformer}& 80.34&69.94&52.79 &67.69&84.37&76.26& 69.29&76.64&85.32&77.76&69.02&\textbf{77.37}\\
ACN~\cite{wang2021acn}
&85.01&71.23&47.39&67.88&86.57&74.65&48.27&\textbf{69.83}&86.68&74.41&47.93&69.67&\\
\midrule
PMKL~\cite{chen2017rethinking}                                         &73.89&	80.86&	77.48 &77.41 &77.46	& 80.71	& 75.40	& \textbf{77.86}    &75.97	&80.35	&76.44	&77.58& \multirow{3}{*}{T1ce} \\
mmFormer~\cite{zhang2022mmformer}&82.23&76.32&67.70 &75.42&83.68&76.09&69.81&76.53&85.50&77.34&70.55&\textbf{77.80}\\
ACN~\cite{wang2021acn}&77.86&67.07&77.63&74.19&78.07&70.05&76.73&74.95&78.09&70.42&78.03&\textbf{75.51}\\
\midrule
PMKL~\cite{chen2017rethinking}                                           &84.09&	66.78&	42.13& 64.33   &83.84	&68.89	&41.41	&64.71    &85.70	& 68.44	& 43.57& 	\textbf{65.90} &  \multirow{3}{*}{Flair} \\
mmFormer~\cite{zhang2022mmformer}&80.45&70.50&54.25 &68.40&83.90& 76.05&65.55&75.17&84.51&76.29&70.23&\textbf{77.01}\\
ACN~\cite{wang2021acn}&85.01&71.23&47.39&67.88&87.16&77.06&46.08&70.10&87.66&76.64&46.61&\textbf{70.30}\\

\bottomrule
\end{tabular}%
\end{table}

As for KD and DA, the performance achieved by student models is dependent on the respective teacher models. Namely, a better teacher will probably provide better guidance to the respective student given the same architectures and knowledge distillation method. We will report the performance of both teachers and students.
For SLS approaches, we only present the performance of the model itself. 
The first line of Table~\ref{tab:ablationp} presents the results of teachers among anchors of $P_N$ and $P_{mix}$s. They will be used to guide the downstream students by KD and DA approaches.
Other lines of Table~\ref{tab:ablationp} show the results obtained by students guided by KD, DA, as well as SLS methods across different anchors. 
It is evident that teachers with latent space aligned to $P_{mix}$ exhibit better segmentation performance, with $81\%$ and $81.29\%$ dices of $P_{mix}$ and $79.13\%$ dice of  $P_{N}$. Aligning to $P_{mix}$ will also achieve superior performance than aligning to $P_N$.

The t-SNE~\cite{van2008visualizing} embeddings depicted in Figure~\ref{fig:ablation_tsne} also demonstrate consistent finding when $\mathcal{T^*}$ is applied, $P_{mix}$ effectively aligns each modality together, as the distribution centers of each modality are almost overlapping. This indicates that the gaps between modalities are narrowed. However, when $P_N$ is used as the anchor, the modality gaps are not sufficiently narrowed, failing to fill the space adequately.
Furthermore, in reference with Figure~\ref{fig:visual_ablation}, it is observed that models with $P_{mix}$ perform better than those with $P_N$.

\subsubsection{$\mathit{P^k_{mix}}$ and $\mathit{P^*_{mix}}$}
In Sec.~\ref{sec:ap}, $P^k_{mix}$ describes the latent space distribution of a specific modality and $P^*_{mix}$ is a weighted combination of all modalities.
We present comparison results when only one modality is available in Table~\ref{tab:ablationp}. Clearly, we can observe that the model with $P^*_{mix}$ achieves better prediction performance than that obtained by $P^k_{mix}$. 
Figure~\ref{fig:visual_ablation} displays segmentation visualization examples, where we can observe that $P^*_{mix}$ consistently obtains the most superior prediction performance among all other anchors.

\subsection{Ablation Studies of Hyperparamters}
\label{sec:Hyperparamters}
\subsubsection{Initialization of $\mathit{P^{*}_{mix}}$}
As discussed in Sec.~\ref{sec:pxin}, initialization of $w_j$ is crucial for the convergence in optimizing the model. 
Since $P^{*}_{mix}$ is an adaptive version of $P^{k}_{mix}$,
we utilize some prior knowledge from $P^{k}_{mix}$ to initialize the parameter. 
Particularly, in this paper, we set 
the initialization weight for T1, T2, T1ce, and Flair to be $0$, $0$, $0$, and $1$ respectively.


\subsubsection{Parameter of alignment weight}
We have 
found that the alignment loss is not sensitive to the hyperparameter. In this paper, we empirically set the alignment loss to one-eighth of the original loss.

\section{Conclusion}
\label{conclusion}
In this paper, we present a novel Medical Modality Alignment Paradigm to mitigate the modality gaps whilst learning simultaneously invariant feature representations in segmenting brain tumors with missing modalities. 
Specifically, we invent an alignment paradigm for the models with an encoder encouraging that all modalities are in the same space and a latent space distribution $P^*_{mix}$ as the aligning anchor.
Meanwhile, we provide theoretical support for the proposed alignment paradigm, demonstrating that individual alignment of each modality to $P_{mix}$ certifies a tighter evidence lower bound than mapping all modalities as a whole.
Extensive experiments have demonstrated the superiority of the proposed paradigm over several latest state-of-the-art approaches, enabling  better segment the brain tumor with missing modalities.

\bibliographystyle{IEEEtran}
\bibliography{tmi}

\end{document}